\documentclass{aamas2013}

\usepackage{epic}
\usepackage{times}
\usepackage{graphicx}
\usepackage{latexsym}
\usepackage{subfigure}
\usepackage{multirow}
\usepackage{amsmath,amssymb}%,amsthm}
\usepackage{float}

\DeclareMathAlphabet\mathbfcal{OMS}{cmsy}{b}{n}

\newtheorem{theorem}{Theorem}

\newtheorem{example}{Example}

\newtheorem{proposition}{Proposition}

\newtheorem{lemma}{Lemma}
\newtheorem{corollary}{Corollary}
\newtheorem{claim}{Claim}

\floatstyle{ruled}
\floatname{algorithm}{Algorithm}
\newfloat{algorithm}{htb}{alg}

\newlength{\halftextwidth}
\title{Coalitional  Manipulation for Schulze's Rule}
\numberofauthors{4}
  \author{
  \alignauthor
  Serge Gaspers\\
	\affaddr{UNSW and NICTA}\\
	\affaddr{Sydney, Australia}\\
	\email{sergeg@cse.unsw.edu.au}
  \alignauthor
  Thomas Kalinowski\\
	\affaddr{University of Rostock}\\
	\affaddr{Rostock, Germany}\\
	\email{thomas.kalinowski@uni-rostock.de}
  \alignauthor
  Nina Narodytska\\
	\affaddr{NICTA and UNSW}\\
	\affaddr{Sydney, Australia}\\
	\email{nina.narodytska@nicta.com.au}
  \and
  \alignauthor
  Toby Walsh\\
	\affaddr{NICTA and UNSW}\\
	\affaddr{Sydney, Australia}\\
	\email{toby.walsh@nicta.com.au}
}

\begin{document}
\maketitle

\begin{abstract}
Schulze's rule is used in
the elections of a large number of organizations
including Wikimedia and Debian. Part of the
reason for its popularity is
the large number of axiomatic properties,
like monotonicity and Condorcet consistency, which
it satisfies. We identify a potential shortcoming of Schulze's
rule: it is computationally vulnerable to manipulation.
In particular, we prove that computing
an unweighted coalitional manipulation (UCM) is polynomial
for any number of manipulators. This result
holds for both the unique winner and the co-winner
versions of UCM. This resolves an open question in~\cite{rankedpairs2}.
We also prove that computing a weighted coalitional manipulation (WCM)
is polynomial for a bounded number of candidates.
Finally, we
discuss the relation between the \emph{unique} winner UCM problem and
the \emph{co-winner} UCM problem and
argue that they
have substantially different necessary and sufficient conditions
for the existence of a successful manipulation.

\end{abstract}

%\begin{category}
%\end{category}


\sloppy

\section{Introduction}\label{sec:intro}

%Voting is a common mechanism to aggregate preferences in multi-agent systems, that
%allows us to identify the socially preferred alternative  from a set of alternatives.
%Many different voting systems have been
%proposed and studied in the literature.
One important issue
with voting is that agents may cast strategic votes instead of
revealing  their true preferences. Gibbard \cite{Gibbard73} and Sattertwhaite \cite{Satterthwaite75}
proved that most voting rules are
manipulable in this way.
Bartholdi, Tovey and Trick~\cite{bartholditoveytrick}
suggested  computational complexity may nevertheless
act as a barrier to manipulation.
%They showed that even if a manipulation exists, it may be computationally too difficult to find.
Interestingly, it
is NP-hard to compute a manipulation
for many commonly used voting rules, including maximin, ranked pairs~\cite{rankedpairs1},
Borda~\cite{dknwaaai11,borda2}, 2nd order Copeland, STV~\cite{stvhard},
Nanson and Baldwin~\cite{nwxaaai11}.
A recent survey on computational complexity as a barrier against manipulation in elections can be found in~\cite{survey1}.
We study here the
computational complexity of manipulating Schulze's voting rule, which
is arguably the most widespread Condorcet voting method in
use today.
%, and the only rule whose complexity to resist manipulation remains unknown.

Schulze's rule was proposed by Markus Schulze in 1997, and was quickly
adopted by many organizations. It is, for example, used by
the Annodex Association, Blitzed, Cassandra,
Debian, the European Democratic Education Conference,
the Free Software Foundation, GNU Privacy Guard, the Haskell
Logo Competition, Knight Foundation,
%the League of Professional System Administrators,
MTV, Neo, Open Stack,
the Pirate Party, RLLMUK, Sugar Labs, TopCoder, Ubuntu
and the Wikimedia Foundation.
In addition to being Condorcet consistent, Schulze's rule
satisfies many other desirable axiomatic properties, including
%. For instance, it satisfies the following criteria:
Pareto optimality, monotonicity,
Condorcet loser criterion, independence to clones,
reversal symmetry and the majority criterion.
Schulze's rule is known by several other
names including the Beatpath Method and Path Voting.
The method can be seen as the inverse procedure  to
another Condorcet consistent voting method, ranked pairs.
The ranked pairs method starts with the
largest defeats and uses as much information about these
defeats as it can without creating ambiguity.
By comparison, Schulze's rule repeatedly removes
the weakest defeat until ambiguity is removed.

Schulze's rule has a number of interesting computational
properties. Whilst it is polynomial to compute the winner
of Schulze's rule,
it requires finding paths in a directed graph labeled with
the strength of defeats. Such paths can be found using
a variant of the cubic time Floyd-Warshall algorithm~\cite{AHUJA93}.
More recently, Parkes and Xia initiated the
study of the computational complexity of manipulating
this voting rule~\cite{rankedpairs2}.
They proved that in the unique winner UCM problem
it is polynomial for a single agent to compute a manipulating
vote if one exists. They also investigated the vulnerability of
Schulze's rule to various types of control.
%, including control by adding and removing alternatives and
%bribery.
However, they left the computational
complexity of UCM with more than one
manipulator as an open question.

In this paper, we continue this study
and show that UCM
remains polynomial for an arbitrary number of manipulators.
For users of Schulze's rule, this result has both positive
and negative consequences. On the negative side,
this means that the rule is computationally vulnerable
to being manipulated.
On the positive side, this means that
when eliciting votes, we can compute in
polynomial time when we have collected enough
votes to declare the winner.
Our results also highlight
the importance of distinguishing carefully between manipulation problems
where we are looking for a single winner
compared to co-winners.

%We prove this result in two steps. First,
%we show that to construct a successful manipulation
%it is sufficient for the manipulators to vote the same way.
%Second, we show that given an input weighted majority
%graph and the number of manipulators we can detect
%in polynomial time if there exists a manipulation
%and construct the manipulators' votes.
%
%We also investigate in details the difference between
%unique winner and co-winner problem
%with a single manipulator. We demonstrate that the result
%form~\cite{rankedpairs2} regarding the complexity
%of the unique winner UCM problem under Schulze's rule
%with a single manipulator cannot be easily extended to
%the co-winner UCM. The main difference is that a manipulation
%always exists in the former case if a trivial necessary condition holds,
%while it might not exist in the latter.
%This highlights
%the importance of distinguishing carefully between manipulation problems
%when we compute a single winner, and co-winners.

\section{Definitions}\label{sec:definitions}
\textbf{Voting systems.}
Consider an election with a set of $m$ candidates $\mathcal C=\{c_1,\ldots,c_m\}$.
%which we identify with the numbers $1,2,\ldots,m$.
A \emph{vote} is specified by a total strict order on $\mathcal C$:
$c_{i_1}\succ c_{i_2}\succ\cdots\succ c_{i_m}$.
%We use lower case letters, e.g. $x,y,c$, to denote a candidate from $\mathcal C$.
%The interpretation is that the preference relation on $\mathcal C'$ is given by this order, every candidate in $\mathcal C'$ is %preferred to every candidate in $\mathcal C'':=\mathcal C\setminus\mathcal C'$
%, and the agent has no preferences among the candidates in $\mathcal C''$. So a vote is given by a sequence
%\[c_1\succ c_2\succ\cdots\succ c_{m'}\succ c_{m'+1}\sim c_{m'+2}\sim\cdots\sim c_{m}\]
%where $m'=\lvert\mathcal C'\rvert$ and $\sim$ denotes indifference.
%Note that we assume that each vote provides a strict order over alternatives. This assumption is not part of the definition of the Schulze method that allows partial order over alternatives~\cite{Schulze2011}.
An $n$-agent profile $P$ on $\mathcal C$ consists of $n$ votes, $P=(\succ_1,\succ_2,\ldots,\succ_n)$.

\textbf{Schulze's voting rule.}
%\[c_{i1}\succ_i c_{i2}\succ_i\cdots\succ_i c_{im}\qquad(i=1,\ldots,n).\]
Given an $n$-agent profile $P$ on $\mathcal C$, Schulze's rule determines a set of winners $\mathcal W_P\subseteq\mathcal C$ as follows.
\begin{enumerate} \itemsep=0pt
\item For candidates $x\neq y$, let $N_P(x,y)$ denote the number of agents who prefer $x$ over $y$, i.e. the number of indices $i$ with $x\succ_i y$.
\item The \emph{weighted majority graph} (WMG) is a directed graph $G_P$ whose vertex set is $\mathcal C$, and with an arc of weight $w_P(x,y)=N_P(x,y)-N_P(y,x)$
for every pair $(x,y)$ of distinct candidates. We denote WMG associated with a profile $P$ by $(G_P,w_P)$.
\item The \emph{strength} of a directed path $\pi=(x_1,x_2,\ldots,x_k)$ in $G_P$ is defined to be the minimum weight over all its arcs, i.e. $w_P(\pi)=\min\limits_{1\leqslant i\leqslant k-1}w_P(x_i,x_{i+1})$.
\item For candidates $x$ and $y$, let $S_P(x,y)$ denote the maximum strength of a path from $x$ to $y$, i.e.
\begin{align*}
 S_P(x,y)=\max\{ & w_P(\pi)\ :\ \pi
 \text{ is a path from $x$ to $y$ in }G_P\}.
\end{align*}
A path from $x$ to $y$ is a \emph{critical path} if its strength is $S_P(x,y)$.
\item The winning set is defined as
\[\mathcal W_P=\left\{x\in\mathcal C\ :\ \forall y\in\mathcal C\setminus\{x\}\ \ S_P(x,y)\geqslant S_P(y,x)\right\}.\]
\end{enumerate}
If $S_P(x,y)>S_P(y,x)$ for two candidates $x,y$, then we say that $x$ dominates $y$. Thus, $\mathcal W_P$ is the set of non-dominated vertices.

The winning set is always non-empty~\cite{Schulze2011}. Note that all weights $w_P(x,y)$, $(x,y) \in G_P$ are  either odd or even, depending on the size of the profile $P$. Conversely, for any weighted digraph where all weights have the same parity, a corresponding profile can be constructed~\cite{Debord87PhD}.
%This result is also widely known as McGarvey's trick~\cite{McGarvey}.
%This procedure was used in the literature.
In the literature, for example,
\cite{rankedpairs1} and~\cite{rankedpairs2} refer
to this as McGarvey's trick~\cite{McGarvey}. We use this
result here as we define the non-manipulators' profile
by their weighted majority graph instead of by
their votes.

 \begin{figure}[tb]
   \centering
   \includegraphics[width=0.65\linewidth]{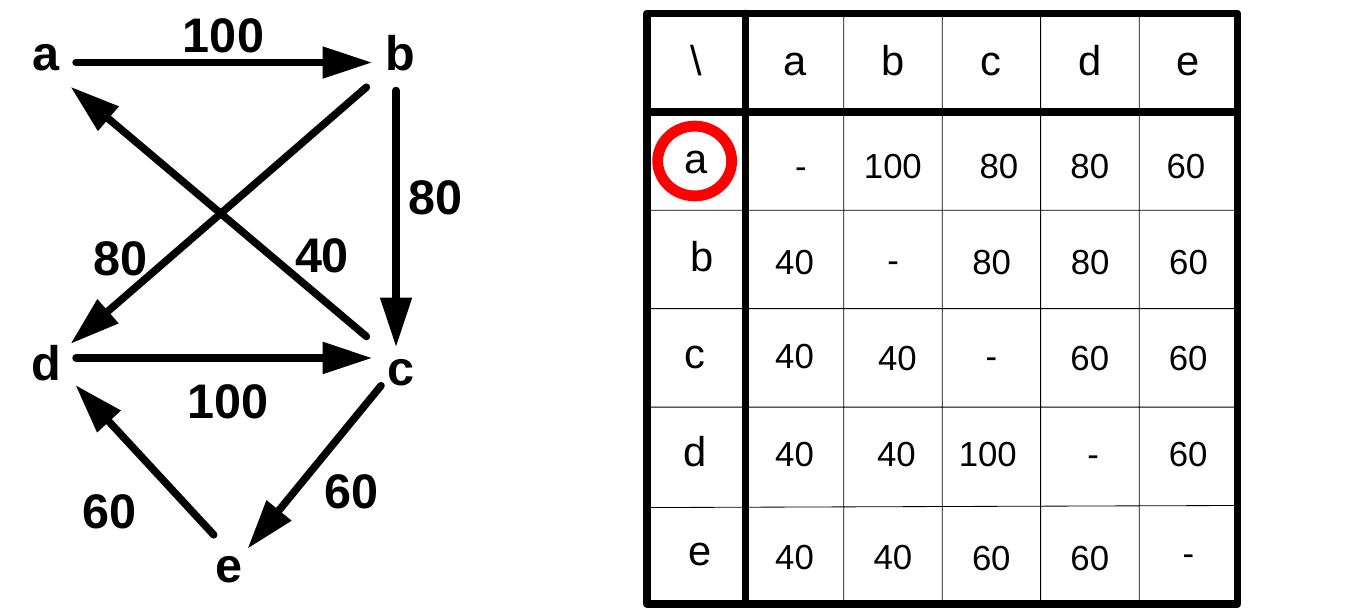}
   \caption{ The weighted majority graph $\mathbf{G_P}$ and the table of $\mathbf{S_P(x,y)}$, $\mathbf{x,y \in \{a,b,c,d,e\}}$ from Example~\ref{exm:intro}.}
   \label{fig:example1}
 \end{figure}

\begin{example}\label{exm:intro}
Consider an election with 5 alternatives $\{a,b,c,d,e\}$.  The weighted majority graph $G_P$ is shown in Figure~\ref{fig:example1}. We omit arcs with zero or negative weight for clarity. The table shows values $S_P(x,y)$, $x,y \in \{a,b,c,d,e\}$. As can be seen from the table, $S_P(a,x) > S_P(x,a)$, for all $x \in \{b,c,d,e\}$. Hence, the winning set contains a single alternative $\mathcal W_P= \{a\}$.\qed
\end{example}

\textbf{Strategic behavior.}
We distinguish between agents that vote truthfully and agents that vote strategically. We call the latter manipulators. We use the superscript $NM$ to denote  the non-manipulators' profile and the superscript $M$ to denote  the manipulators' profile.
The co-winner \emph{unweighted   coalitional manipulation} (UCM) problem is defined as follows. An instance is a tuple $(P^{NM},c,M)$, where $P^{NM}$ is the non-manipulators' profile, $c$ is the candidate preferred by the manipulators and $M$ is the set of manipulators. We are asked whether there exists a profile $P^{M}$ for the manipulators such that $c \in \mathcal W_{P^{NM}\cup P^{M}}$. The \emph{unique} winner UCM problem is a variant of the co-winner UCM where we are looking for
a manipulation such that $\{c\}= W_{P^{NM}\cup P^{M}}$.
 The {\em weighted  coalitional manipulation} (WCM) is defined similarly, where the weights of the agents (both non-manipulators and manipulators) are integers and are also given as inputs. %In a  \controladdvotes{} problem, we are given $n$ votes and a set  of additional votes $V'$, $|V'| \geqslant k$ and a quota $k$. The problem is to determine if there exists a manipulation such that the manipulator adds at most $k$ votes from the additional set of votes $V'$.
%As is common in the literature, we break ties in favour of the
%coalition of the manipulators in the co-winner UCM/WCM.

\section{Weighted Coalitional Manipulation}\label{sec:wcm}
{We consider the co-winner WCM problem for Schulze's voting rule. We show that if there exists a successful manipulation $P^M$ then there exists a successful manipulation $P^{M'}$ where all manipulators vote identically. We prove this result in two steps. First, we construct a kind of directed spanning tree of the WMG $G_{P^{NM} \cup P^M}$  rooted at $c$, which gives us a critical path from $c$ to all other alternatives. Then, by traversing this tree, we build a new linear order of candidates that specifies a vote for all manipulators.}

 \begin{figure}[tb]
   \centering
   \includegraphics[width=0.8\linewidth]{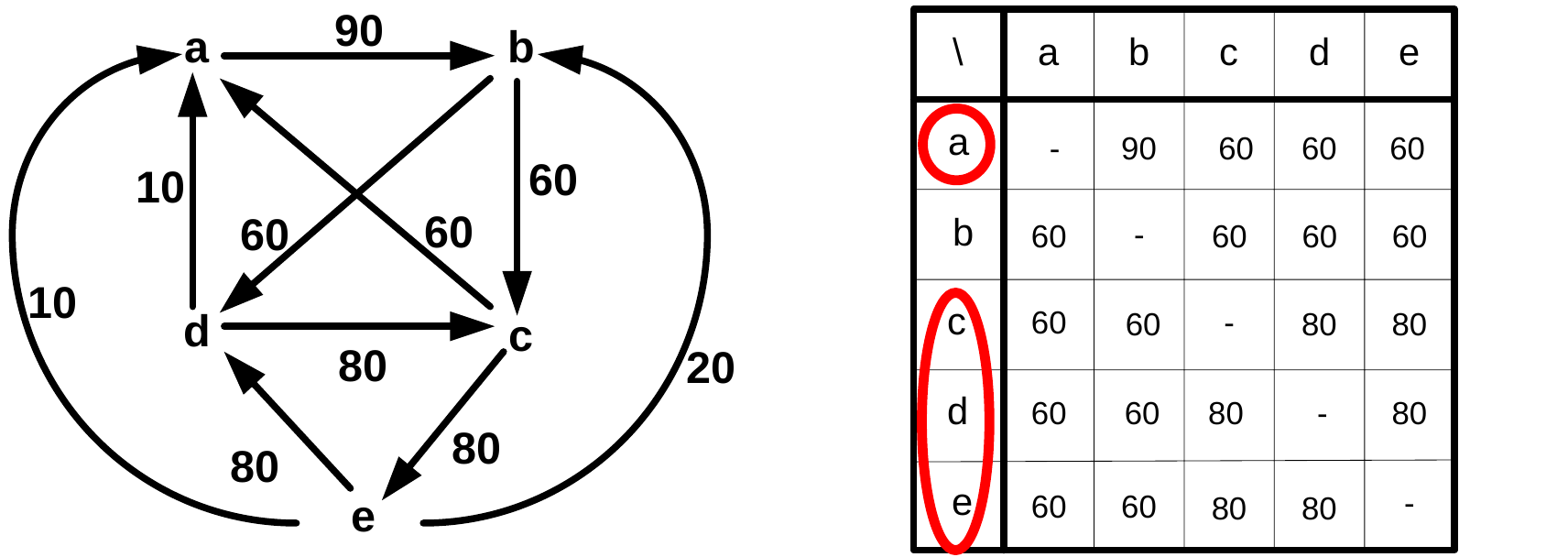}
   \caption{  The WMG and the table of $\mathbf{S_{P^{NM} \cup P^{M}}(x,y)}$, $\mathbf{x,y \in \{a,b,c,d,e\}}$ from Example~\ref{exm:intro1}. }
   \label{fig:example2_1}
 \end{figure}

\begin{example}\label{exm:intro1}
{Consider the WMG $G_{P}$ from Example~\ref{exm:intro}. Suppose that $P$ corresponds to the non-manipulators' profile, so that $P^{NM} = P$. Suppose we have 4 manipulators with weights $10,3,2$ and $5$ that vote in the following way: the first three manipulators vote $c \succ e \succ d \succ b \succ a$ and the last manipulator votes $c \succ a \succ e \succ d \succ b$. Hence, the total weight of the vote $c \succ e \succ d \succ b \succ a$ in $P^M$ is $15$ and  the total weight of the vote $c \succ a \succ e \succ d \succ b$  in $P^M$ is $5$. The updated WMG $G_{P^{NM} \cup P^{M}}$ and the corresponding table that shows the values of pairwise maximum strengths are shown in Figure~\ref{fig:example2_1}. Note that the alternative $c$ is non-dominated as well as alternatives $\{a,d,e\}$. Hence, the winning set $\mathcal W_{P^{NM} \cup P^{M}}= \{a,c,d,e\}$.}\qed
\end{example}

%\subsection{Homogenizing vote sets}\label{sec:homog}
%In this section
We show that given any profile $P$,
%= P^{NM} \cup P^{M}$,
a winning candidate $c\in \mathcal W_P$ and a subset $P_0$ of the set of votes, e.g. $P_0 = P^M$, we can modify the votes in $P_0$ to be all the same, and $c$ is still in the winning set of the resulting profile $P'$. To do this, we construct a vote
$\Lambda=(c\succ c_1\succ\cdots\succ c_{m-1})$
such that $c$ is still a winner if we replace every vote in $P_0$ by $\Lambda$.
{Hence, in the context of the manipulation problem  we can think of  $P $ as $ P^{NM} \cup P^{M}$ and $P_0$ as $P^{M}$.}

An \emph{out-branching} $T$ of a directed graph $G$ rooted at a vertex $r$ is a connected spanning subdigraph of $G$ in which $r$ has in-degree $0$ and all other vertices have in-degree $1$.

\begin{lemma}\label{lem:spanning_tree}
Let $G=G_P$ be the digraph associated with the given profile $P$. There exists an out-branching $T$ rooted at $c$ in $G$ such that for every candidate $c'\neq c$ the unique path from $c$ to $c'$ in $T$ is a critical $c$-$c'$-path in $G$.
\end{lemma}
\begin{proof}
We construct an out-branching $T$ of $G$ by Algorithm \ref{alg:spanning_tree}.
At the initial step the algorithm makes
$c$ the root of $T$. At each step, we add a new vertex $x$, $x \in V(G) \setminus V(T)$,
to the tree $T$ iff the arc $(x,y), y \in V(T)$ has maximum value
$w(x,y)$ among all arcs $(x',y')$, $x \in V(G) \setminus V(T), y \in V(T)$.
%Note that the out-branching is not unique.
  \begin{algorithm}
  \vspace*{-0.4cm}
\caption{Out-branching construction}\label{alg:spanning_tree}
    \begin{tabbing}
          .....\=.....\=.....\=.....\=.....\=.............................. \kill \\
\textbf{Input:}\>\>\> a weighted digraph $(G,w)=(G_P,w_P)$ and\\
\>\>\> a distinguished candidate $c$.\\[1ex]
Initialize $F_1 = \{c\}$, $X_1 = \mathcal C \setminus \{c\}$ and $T_1 = \{\}$.\\
\textbf{for} $i=1,\ldots,m-1$ \textbf{do}\\
\> $D=\max\{w(x,y)\ :\ x\in F_i,\ y\in X_i\}$\\
\> Choose  $a\in F_i$ and $b\in X_i$ with $w(a,b)=D$\\
\> $F_{i+1}=F_i \cup \{b\}$\\
\> $X_{i+1} = X_i \setminus \{b\} $\\
\> $T_{i+1} = T_i \cup \{(a,b)\}$\\
return $T=T_{m}$
    \end{tabbing}
    \vspace*{-0.4cm}
  \end{algorithm}

Clearly, Algorithm \ref{alg:spanning_tree} returns an out-branching because the input digraph is complete. So we just have to show that it satisfies the required property. We do this by induction on the size of $T$. For $i=1$ the claim is obvious, so assume $1<i<m$, and let $b$ be the vertex added in step $i$, i.e. $\{b\}=F_{i+1}\setminus F_i$. Let $\pi=(c=a_0,a_1,\ldots,a_{k-1},a_k=b)$ be the $c$-$b$-path in $T$, and let $j$ be the index of the first arc on that path realizing its strength, i.e.
$j=\min\{t\ :\ w(a_t,a_{t+1})=w(\pi)\}$.
Let $q$ be the step in which the arc $(a_t,a_{t+1})$ is added to $T$. Now suppose that there is a $c$-$b$-path $\pi'=(c,f_1,\ldots,f_r,\ldots,b)$ in $G$ with $w(\pi')>w(\pi)$. Because $c\in\pi'$ and $\pi'\not\subseteq T$, there exists some arc $(f_r,f_{r+1})\in\pi'$ with $f_r\in F_q$ and  $f_{r+1}\in X_q$. Then,
$w(\pi)=w(a_t,a_{t+1})\geqslant w(f_r,f_{r+1})\geqslant w(\pi')$,
contradicting the assumption and thus concluding the proof.
\end{proof}
  \begin{figure}[tb]
   \centering
   \includegraphics[width=0.8\linewidth]{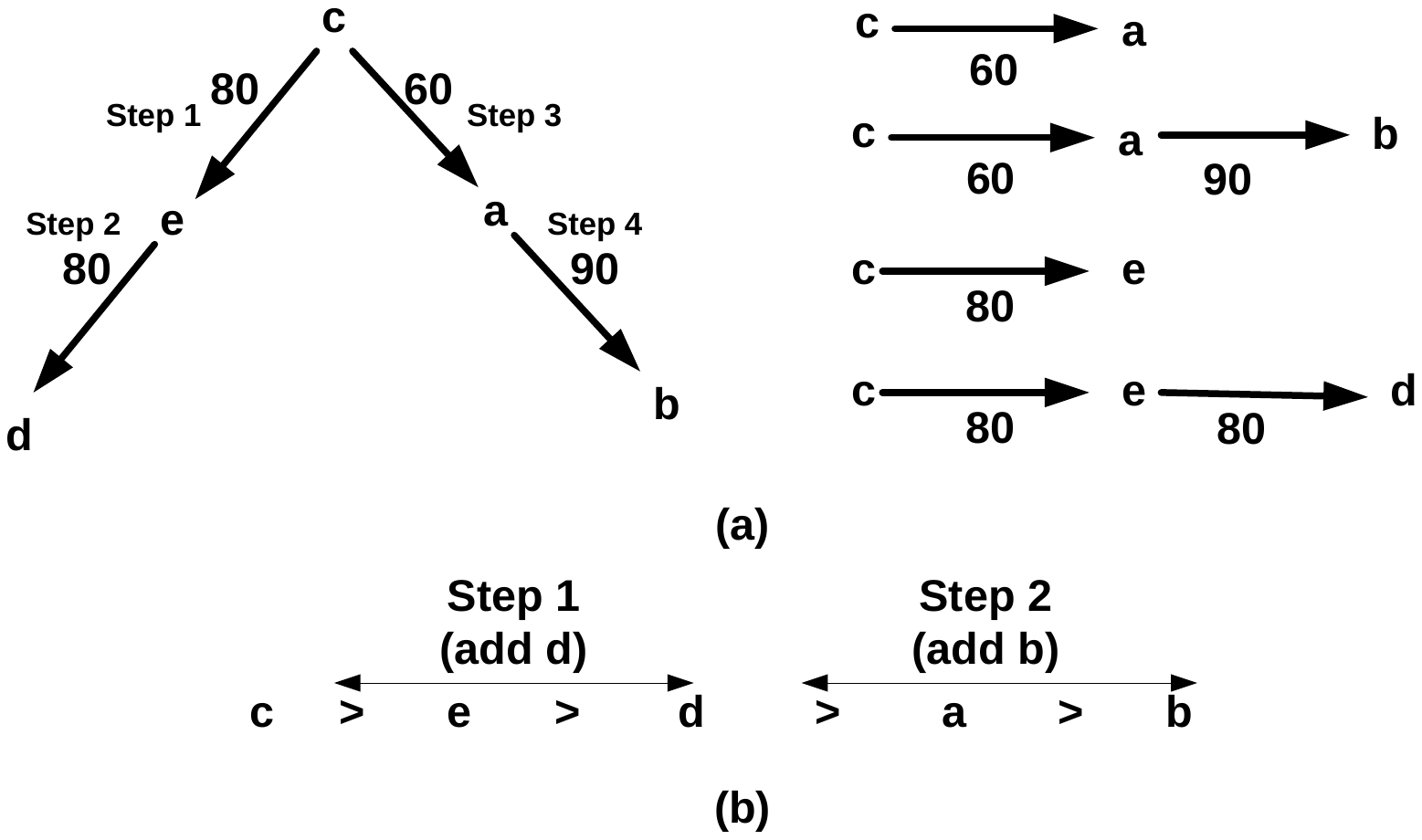}
   \caption{ (a) The out-branching rooted at $\mathbf{c}$ that is produced by Algorithm~\ref{alg:spanning_tree}
   and the corresponding critical $\mathbf{c}$-$\mathbf{x}$-paths, $\mathbf{x \in \{a,b,d,e\}}$;
   (b) The $\mathbf{\Lambda}$ ordering constructed by Algorithm~\ref{alg:ordering}. }
   \label{fig:example2_2}
 \end{figure}
\begin{example}\label{exm:intro_spanning}{
Figure~\ref{fig:example2_2}(a) shows the out-branching for $G_{P^{NM} \cup P^{M}}$ and critical $c$-$x$-paths, $x \in \{a,b,d,e\}$, of the WMG from Example~\ref{exm:intro1}. Consider, for example, the path $(c,e,d)$ in the out-branching. This path has strength $80$ and it corresponds to the maximum strength $c$-$d$-path in $G_{P^{NM} \cup P^{M}}$.}\qed
\end{example}

\begin{lemma}\label{lem:ordering}
Let $G=G_P$ be the graph associated with the given profile $P$ and let $T$ be an out-branching rooted at $c$ as in Lemma \ref{lem:spanning_tree}. Then there exists an ordering
$\Lambda=(c\succ c_1\succ\cdots\succ c_{m-1})$
on the set of candidates with the following properties.
\begin{itemize}
\item Property 1: For each $c_i$ the unique $c$-$c_i$-path in $T$ respects the ordering $\Lambda$, i.e. it is of the form
$(c,c_{j_1},\ldots,c_{j_k}=c_i) \text{ with } 1\leqslant j_1< j_2<\cdots<j_k$.
\item Property 2: The strength of a critical path from $c_i$, $i \in [1,m)$ to $c$ is  nonincreasing along the ordering $\Lambda$:
\[S_P(c_i,c)\geqslant S_P(c_j,c)\qquad\text{for }1\leqslant i<j\leqslant m-1.\]
\end{itemize}
\end{lemma}
The intuition for Property 1 is that the strength of each critical path from $c$ to $c_i$, $i \in [1,m)$ does not decrease if we change all votes in $P_0$ to $\Lambda$.
\begin{proof}
Algorithm \ref{alg:ordering} returns a total order on the set of candidates.
The algorithm traverses the out-branching $T$  obtained by Algorithm~\ref{alg:spanning_tree}.
At each step, we identify a vertex $x$ with the largest value of the strength $S_P(x,c)$.
Then we find the path $\pi$ from $c$ to $x$ in $T$ which is a critical path by Lemma~\ref{lem:spanning_tree}.
A prefix of the path $\pi$ might be added to $\Lambda$ at this point. Hence, we only focus
on the suffix of $\pi$ that does not contain vertices added to $\Lambda$.
Then we add the vertices in this suffix of $\pi$ to $\Lambda$ in the order in which
they appear in $\pi$. We terminate when $\Lambda$ is a total order over all alternatives.

  \begin{algorithm}
\caption{Construction of the ordering $\Lambda$}\label{alg:ordering}
\vspace*{-0.4cm}
    \begin{tabbing}
          .....\=.....\=.....\=.....\=.....\=.............................. \kill \\
\textbf{Input:} \>\>\> a weighted digraph $(G,w)=(G_P,w_P)$,\\
\>\>\> a distinguished candidate $c$ and\\
\>\>\> the out-branching $T$ with root $c$ from Algorithm \ref{alg:spanning_tree}.\\[1ex]
Initialize $\Lambda=(c)$, $X=\mathcal C\setminus\{c\}$\\
\textbf{while} $X\neq\varnothing$ \textbf{do} \\
\> $D=\max\{S_P(x,c)\ :\ x\in X\}$\\
\> Let $a\in X$ be any vertex with $S_P(a,c)=D$.\\
\> Let $\pi$ be the unique $c$-$a$-path in $T$.\\
\> Add the vertices in $\pi\cap X$ to $\Lambda$ in the order in which \\
\>they appear on $\pi$.\\
\> Update $X:=X\setminus\pi$.\\
return $\Lambda$
     \end{tabbing}
\vspace*{-0.4cm}
\end{algorithm}
We show that it satisfies the two properties by induction on the length of $\Lambda$. For the initial $\Lambda=(c)$ it is obviously true. So suppose we are in the while loop adding $\pi\cap X$ for a $c$-$a$-path
$\pi=(c=g_0,g_1,\ldots,g_k,a)$.
Note that $\pi\cap X$ is a suffix of $\pi$, i.e. $\pi\cap X=\{g_{j+1},\ldots,g_k,a\}$ for some $j$. To see this, let $g_j$ be the last vertex on $\pi$ that is already in $\Lambda$. Then by construction, all the vertices $g_1,g_2,\ldots,g_j$ have been added to $\Lambda$ in the step in which $g_j$ was added or earlier.

By the induction hypothesis the $c$-$g_j$-path in $T$ respects $\Lambda$, and because the suffix $g_{j+1},\ldots,g_k,a$ is added to $\Lambda$ and  $g_{j+1},\ldots,g_k,a$ is a sub-path of $\pi$, the condition of Property 1 is satisfied for all these vertices.

Next we observe that $S_P(g_t,c)\geqslant S_P(a,c)$ for $t=j+1,\ldots,k$. To see this, let $\pi'$ be an $a$-$c$-path of strength $S_P(a,c)$. We have $w(g_t,g_{t+1})\geqslant S_P(c,a)\geqslant S_P(a,c)$
 for all $t$, where the first inequality is true because $(g_t,g_{t+1})$ is an arc on the $c$-$a$-path in $T$ which is a critical path, and the second inequality because $c$ is a winner. Thus the concatenation of $g_t,g_{t+1},\ldots,g_k,a$ and $\pi'$ provides a $g_t$-$c$-path of strength $S_P(a,c)$. Now Property 2 follows from the observation that $S_P(x,c)\leqslant S_P(a,c)$ for all $x\in X\setminus\pi$ which follows from the maximality condition in the step where $a$ is chosen.
\end{proof}

  \begin{figure}[tb]
   \centering
   \includegraphics[width=0.8\linewidth]{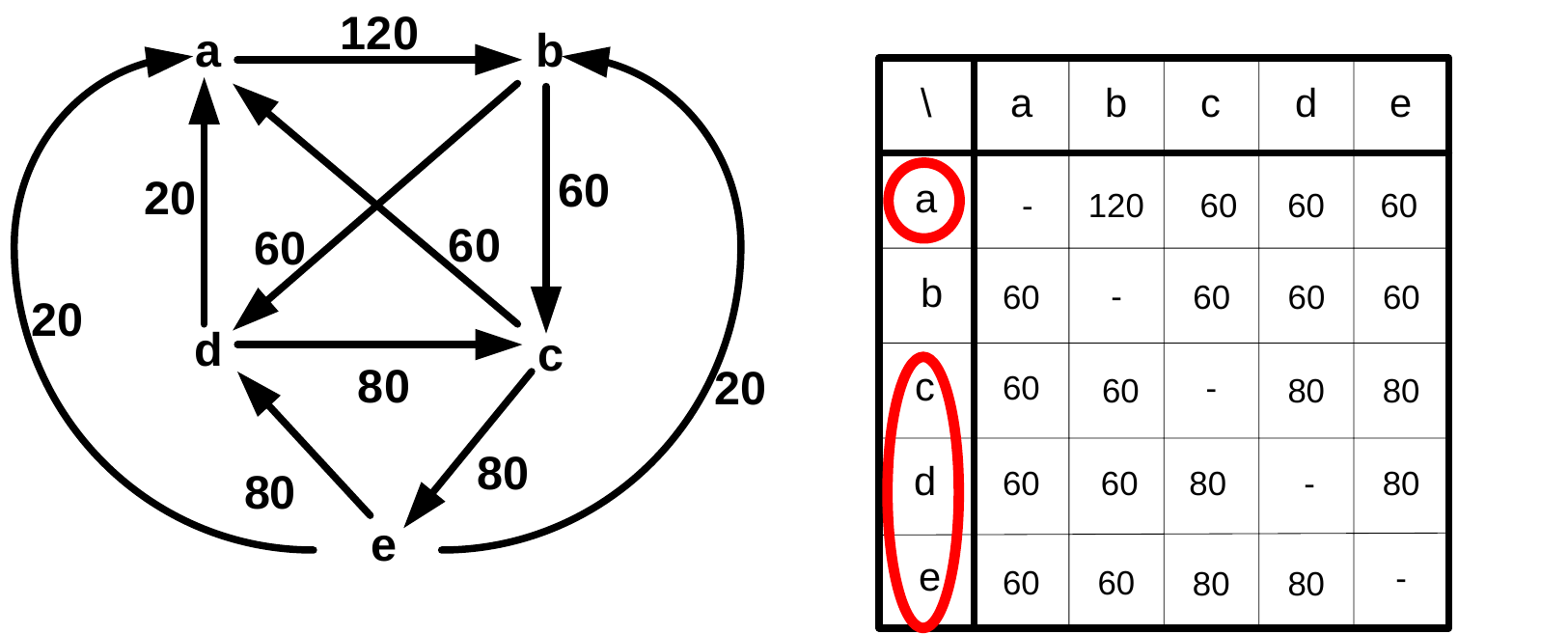}
   \caption{ An updated WMG and the table of $\mathbf{S_{P^{NM} \cup P^{M}}(x,y)}$, $\mathbf{x,y \in \{a,b,c,d,e\}}$ from Example~\ref{exm:intro_new_wmg} using the $\mathbf{\Lambda}$ ordering constructed by Algorithm~\ref{alg:ordering}. }
   \label{fig:example2_3}
 \end{figure}

\begin{example}\label{exm:intro_new_wmg}
We construct an ordering $\Lambda$ based on the out-branching obtained in Example~\ref{exm:intro_spanning}. The alternatives $\{d,e\}$ are such that $ S_{P^{NM} \cup P^M}(e,c) = S_{P^{NM} \cup P^M}(d,c) = \max\{S_{P^{NM} \cup P^M}(x,c)\ :\ x\in \{a,b,d\}\}$. We break the tie between $e$ and $d$ arbitrarily and select $d$. Hence, we build a partial order $c \succ e \succ d$. The next alternatives that we consider are $\{b,a\}$ as $ S_{P^{NM} \cup P^M}(b,c) = S_{P^{NM} \cup P^M}(b,a) = \max\{S_{P^{NM} \cup P^M}(x,c)\ :\ x\in \{a,b\}\}$. We select $b$ and add the suffix $a \succ b$ to the partial order $c \succ e \succ d$, so that we get  $\Lambda = (c \succ e \succ d \succ a \succ b)$. Hence, $4$ manipulators can vote with respect to $\Lambda$. Figure~\ref{fig:example2_2}(b) shows the execution of Algorithm~\ref{alg:ordering}.
Figure~\ref{fig:example2_3} shows the new WMG and the corresponding table of maximum strengths. It is easy to see that $c$ is still a winner after the manipulators change their votes.\qed
\end{example}

For our given profile $P$ and distinguished candidate $c$, we construct an ordering $\Lambda$ as described in the proof of Lemma \ref{lem:ordering}.
\begin{theorem}\label{thm:homogeneity}
Let $P$ be any profile with candidate $c$ in the winning set, let $P_0\subseteq P$ be any subprofile, and set $P_1=P\setminus P_0$. Let $P'$ be the profile given by $P'=P_1\cup \bigcup_{i=1}^{\lvert P_0\rvert}\{\Lambda\}$,
where $\Lambda$ is the ordering constructed in Lemma \ref{lem:ordering}. Then $c$ is still in the winning set $\mathcal W_{P'}$.
\end{theorem}
\begin{proof}
Denote the WMGs associated with the two profiles by $(G,w)=(G_P,w_P)$ and $(G',w')=(G_{P'},w_{P'})$.
We recall that we use the out-branching $T$ with root $c$ obtained by Algorithm \ref{alg:spanning_tree}. The theorem is based on the following two claims.

\begin{claim}\label{cl:1}
 For each path $\pi$ in $T$ starting from $c$ the strength of $\pi$ does not decrease in the graph $G'$, i.e. $w'(\pi)\geqslant w(\pi)$.
\end{claim}
By construction of $\Lambda$, we have $w'(x,y)\geqslant w(x,y)$ for every arc $(x,y)\in T$, and this implies Claim~\ref{cl:1}.

\begin{claim}\label{cl:2}
 For every $a$-$c$-path $\pi$, the strength of $\pi$ in $G'$ does not exceed the strength of a critical $a$-$c$-path in $G$, i.e. $w'(\pi)\leqslant S_P(a,c)$.
\end{claim}
 To prove Claim~\ref{cl:2}, assume, for the sake of contradiction, that $a$ is a vertex such that there is an $a$-$c$-path $\pi=(a=a_1,\ldots,a_k=c)$ with $w'(\pi)> S_P(a,c)$, and w.l.o.g. we assume that for all $a_i$-$c$-paths $\sigma$, $1\le i\le k-1$, we have $w'(\sigma)\leqslant S_P(a_i,c)$. Because $c$ is a winner with respect to $P$, $\pi$ must contain an arc $(x,y)$ of weight $w(x,y)$ such that $w(x,y) \leqslant S_P(c,a)$. Let $(b,d)=(a_i,a_{i+1})$ be the first arc with this property, i.e. $i=\min\{j\ :\ w(a_j,a_{j+1})\leqslant S_P(a,c)\}$. Next we show the chain of inequalities
\[w'(\pi)\stackrel{(1)}{>}S_P(a,c)\stackrel{(2)}{\geqslant} S_P(b,c)\stackrel{(3)}{\geqslant} S_P(d,c)\stackrel{(4)}{\geqslant} S_{P'}(d,c)\stackrel{(5)}{\geqslant} w'(\pi),\]
 which is a contradiction and thus proves the claim.

 \begin{figure}[tb]
   \centering
   \includegraphics[width=.5\linewidth]{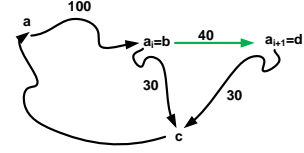}
   \caption{A diagram illustrating the arguments for steps (2),(3) and (4) of the inequality chain in the proof of Claim 2.}
   \label{fig:claim_2}
 \end{figure}

 The following arguments for the single inequalities above are illustrated in Figure \ref{fig:claim_2}.
\begin{description}
\itemsep=0pt
\item[(1)] By assumption.
\item[(2)] As $c$ is a winner for $P$, every $a$-$c$-path must contain an arc $(x,y)$ with $w(x,y)\leqslant S_P(c,a)$. By the choice of $b$, we know that $(b,d)$ is the first arc such that $w(b,d)\leqslant S_P(a,c)$. Hence, the strength
    of the $a$-$b$-path  is greater than the strength of the $a$-$c$-path, $S_P(a,b)>S_P(a,c)$. Now from
$S_P(a,c)\geqslant\min\{S_P(a,b),S_P(b,c)\}$
it follows that $S_P(b,c)\leqslant S_P(a,c)$.
\item[(3)] From the assumption $w'(\pi)> S_P(a,c)$ it follows that $w'(b,d)>w(b,d)$ which implies that $b$ comes before $d$ in the ordering $\Lambda$, and then the inequality (3) follows from Lemma \ref{lem:ordering}.
\item[(4)] By assumption, $w'(\sigma)\leqslant S_P(d,c)$ for all $d$-$c$-paths $\sigma$, hence $S_{P'}(d,c)\leqslant S_P(d,c)$.
\item[(5)] Let $\pi_1$ be the $d$-$c$-subpath of $\pi$. Then $S_{P'}(d,c)\geqslant w'(\pi_1)\geqslant w'(\pi)$.
\end{description}
Together, Claims~\ref{cl:1} and \ref{cl:2} prove the theorem.
\end{proof}

\begin{corollary}\label{col:wcm}
The co-winner WCM problem for Schulze's rule is polynomial if the number of candidates is bounded.
\end{corollary}
\begin{proof}
As the number of candidates is bounded we can enumerate all possible distinct votes in polynomial time. From Theorem~\ref{thm:homogeneity} it follows that it is sufficient to consider manipulations where all manipulators vote identically.
\end{proof}

%\paragraph{Problem definition}
%We distinguish between agents that vote truthfully and agents that vote strategically. We call the latter manipulators. We use the superscript $NM$ to denote the non-manipulators' profile and the superscript $M$ to denote the manipulators' profile.
%
%The co-winner \emph{unweighted coalitional manipulation} (UCM) problem is defined as follows. An instance is a tuple $(P^{NM},M)$, where $P^{NM}$ is the non-manipulators' profile and $M$ is the set of manipulators. Without loss of generality, we assume that the manipulators prefer candidate $1$. The question is whether there exists a profile $P^{M}$ for the manipulators such that $1 \in \mathcal W_{P^{NM}\cup P^{M}}$.
%The \emph{unique} winner UCM problem is a variant of the co-winner UCM where we are looking for a manipulation such that $\{c\}= W_{P^{NM}\cup P^{M}}$.
%The {\em weighted  coalitional manipulation} (WCM) is defined similarly, where the weights of the agents (both non-manipulators and manipulators) are also given as inputs.

\section{Unweighted Coalitional Manipulation}

%Let $G = G_{P^{NM}}$ denote the weighted majority graph of $P^{NM}$ with arc weight function $w$.
%Let $S(x,y)$ denote the maximum strength of a path from $x$ to $y$ in $G$.
In this section we present our main result: co-winner UCM is polynomial for any number of manipulators.
This closes an open question raised in~\cite{rankedpairs2}.
By Theorem~\ref{thm:homogeneity}, $(P^{NM},c,M)$ is a Yes-instance for co-winner UCM if and only if there is a vote $\succ'$ such that $c\in \mathcal W_{P^{NM} \cup P^{M}}$ where votes in $P^{M}$ corresponds to $\succ'$.
%Thus,
It remains to decide if such a vote $\succ'$ exists.

As in the weighted case, we denote $(G,w)=(G_P,w_P)$ and $(G',w')=(G_{P'},w_{P'})$ the WMGs of the voting profiles $P = P^{NM}$ and $P'=P^{NM} \cup P^{M}$ with arc weight functions $w$ and $w'$, respectively,  and  $S_{P'}(x,y)$ denotes the maximum strength of a path from $x$ to $y$ in $G'$.

First, we give a high-level description of the two-stage algorithm.
In the first stage, we run a preprocessing procedure on $G$
that aims to identify a set of necessary constraints on the strengths $S_{P'}(x,y)$, such as
$S_{P'}(x,y)$ must be equal to $S_P(x,y)+|M|$. The procedure is based on a set of rules
that enforce necessary conditions for $c$ to win, namely, $S_{P'}(c,x) \geqslant S_{P'}(x,c)$ must hold.
If the preprocessing procedure detects a failure then there is no set of votes for $M$ such that $c$ becomes a winner.
The pseudocode for the first stage of the algorithm is given in Algorithm~\ref{alg:bounds}.
Section~\ref{ss:first_stage} proves the correctness of Algorithm~\ref{alg:bounds}.
If no failure is detected by applying these rules during the preprocessing stage, we show
that a manipulation exists and provide a constructive procedure that finds a manipulation.
The pseudocode for the second stage of the algorithm is given in Algorithm~\ref{alg:spanning_tree_ucm}.
Here, the algorithm traverses vertices in $G$ in a specific order, which defines
the manipulators votes. Section~\ref{ss:second_stage} proves the correctness of Algorithm~\ref{alg:spanning_tree_ucm}.

\begin{algorithm}[tb]
\caption{\textsc{PreprocessingBounds}.}\label{alg:bounds}
\vspace*{-0.4cm}
  \begin{tabbing}
    .....\=.....\=.....\=.....\=.....\=.....\=.............. \kill \\
\textbf{Input:}  \> \> \> a weighted digraph $(G=(V,E),w)=(G_P,w_P)$, \\
\> \> \> the strengths $S_P$ and \\
\> \> \> a distinguished candidate $c$.\\[1ex]
\textbf{for} $(x,y)\in V\times V$ \textbf{do}\\
\> $\overline{w}(x,y)= w(x,y)+|M|$ \\
\> $\underline{w}(x,y)= w(x,y)-|M|$\\
\> $U(x,y)= S_P(x,y)+|M|$\\
\textbf{while} no convergence \textbf{do}\\
\> /*  Rule 1  */\\
\> \textbf{for} $x\in V\setminus\{c\}$ \textbf{do}\\
\> \>  $U(x,c)=\min\{U(x,c),\ U(c,x)\}$ \\
\> /*  Rule 2  */\\
\> \textbf{for} $x \in V\setminus\{c\}$ \textbf{do}\\
\> \>  $V_{r} = \left\{y \in V : U(y,c) < U(x,c), y \neq c\right\}$ \\
\> \>  $E_{r} = \left\{(f,g) \in E :\ \overline w(f,g)<U(x,c)\right\} \cup V_r \times V \cup V \times V_r$ \\
\> \>  $G^x=\Big((V\setminus V_r), (E \setminus E_r)\Big)$\\
\> \>  \textbf{if} $G^x$ contains no $c$-$x$-path \textbf{then}\\
\> \>  \> $U(x,c)= U(x,c) -2$\\
\> /*  Rule 3  */\\
\> \textbf{for} $x \in V\setminus\{c\}$ \textbf{do}\\
\> \> \textbf{for} $y \in V\setminus\{x,c\}$ \textbf{do}\\
\> \> \> \textbf{if} $U(x,c)< \underline w(x,y) $\\
\> \> \> \> $U(y,c)= \min(U(y,c), U(x,c))$\\
\> \textbf{for} $x \in V\setminus\{c\}$ \textbf{do}\\
\> \> \textbf{if} $U(x,c) < S_P(x,c)-|M|$ \textbf{then}\\
\> \> \> return \textbf{FAIL}\\
return $U$
\end{tabbing}
\vspace*{-0.4cm}
\end{algorithm}

\subsection{Stage 1. Preprocessing}\label{ss:first_stage}
Algorithm~\ref{alg:bounds} uses a function $U(x,y)$, which for any two candidates $x$ and $y$, gives an upper bound for $S_{P'}(x,y)$.
Initially,
%\begin{align*}
$U(x,y) := S_P(x,y) + |M|$
%\end{align*}
for each pair %of candidates
$(x,y)$. We also use the following notation
for an upper and lower bound of $w'(x,y)$:
%\begin{align*}
 $\overline w(x,y)  := w(x,y) + |M|$   and
 $\underline w(x,y) := w(x,y) - |M|$.
%\end{align*}

In the first stage, Algorithm~\ref{alg:bounds} decreases $U(x,y)$ when it detects necessary conditions implying $S_{P'}(x,y)<U(x,y)$.
The algorithm is based on the following three reduction rules.
We show that these rules are sound in the sense that an application of a rule does not change
the set of solutions of the problem.
%The following three reduction rules will make these necessary conditions precise.
\begin{description}
 \item[Rule 1.] If there is a candidate $x$ such that $U(c,x) < U(x,c)$, then set $U(x,c) := U(c,x)$.
\end{description}
\begin{proposition}
Rule 1 is sound.
\end{proposition}
\begin{proof}
To see that Rule 1 is sound, suppose $S_{P'}(x,c) > S_{P'}(c,x)$. But then, $c \notin W_{P'}$.%W_{P^{NM} \cup P^{M}}$.
\end{proof}

To state the next reduction rule, define for any candidate $x$ the directed graph $G^x$ obtained from
$G$ by removing all vertices $y$ with $U(y,c)<U(x,c)$ and all arcs $(y,z)$ such that $\overline w(y,z) < U(x,c)$.

\begin{description}
 \item[Rule 2.] If there is a candidate $x$ such that $G^x$ has no directed path from $c$ to $x$, then set $U(x,c) := U(x,c)-2$.
\end{description}
\begin{proposition}
Rule 2 is sound.
\end{proposition}
\begin{proof}
Suppose the premises of the rule hold, and, for the sake of contradiction, suppose there exists a path in $G'$ from $x$ to $c$ with strength
$s$, where $s$ equals $U(x,c)$ before the application of the rule.
Since $G^x$ has no directed path from $c$ to $x$,
all directed paths in $G$ from $c$ to $x$ pass either through a vertex $y$ with $U(y,c)<s$ or through an arc $(y,z)$ such that $\overline w(y,z) < s$. Since any such path  has strength less than $s$, we have that $S_{P'}(c,x)<s$.
But, since $c$ belongs to the winning set in $G'$, we have that $S_{P'}(c,x) \geqslant S_{P'}(x,c) \geqslant s$, a contradiction.
Thus, $S_{P'}(x,c) < s$.
The soundness of Rule 2 now follows from the fact that all $S_{P'}(y,z)$ have the same parity as $|NM|+|M|$, $y,z\in V$, and we maintain the invariant that all $U(\cdot, \cdot)$ have this parity.
\end{proof}

\begin{description}
 \item[Rule 3.] If there are candidates $x,y \neq c$ such that $U(x,c) < \underline w(x,y)$ and $U(y,c)>U(x,c)$, then set $U(y,c):=U(x,c)$.
\end{description}
\begin{proposition}
Rule 3 is sound.
\end{proposition}
\begin{proof}
Suppose $S_{P'}(y,c) > U(x,c)$ and $\pi$ is a critical path from $y$ to $c$ in $G'$. But then, the path $x\cdot\pi$, obtained by concatenating $x$ and $\pi$, has
strength $\min \{w'(x,y), S_{P'}(y,c)\}$. Since $w'(x,y)\geqslant \underline w(x,y) > U(x,c)$, the strength of this directed path from $x$ to $c$ is strictly greater than $U(x,c)$, contradicting our assumption that $U(x,c)$ is a necessary upper
bound for $S_{P'}(x,c)$.
\end{proof}

We remark that Rules 1-3 decrement $U(\cdot,c)$ when necessary conditions are found that require a smaller upper bound for $S_{P'}(\cdot,c)$.
Should at any time such a value $U(x,c)$ become smaller than $S_P(x,c)-|M|$, then there are no votes for $M$ that make $c$ a winner.
In this case, the preprocessing algorithm returns \textbf{FAIL}.

\begin{theorem}
Algorithm~\ref{alg:bounds} is sound.
\end{theorem}
\begin{proof}
Algorithm~\ref{alg:bounds} implements  Rules~1--3. As these rules are sound, the algorithm is sound.
\end{proof}

Consider how Algorithm~\ref{alg:bounds} works on an example.

\begin{figure}[tb]
   \centering
   \includegraphics[width=0.6\linewidth]{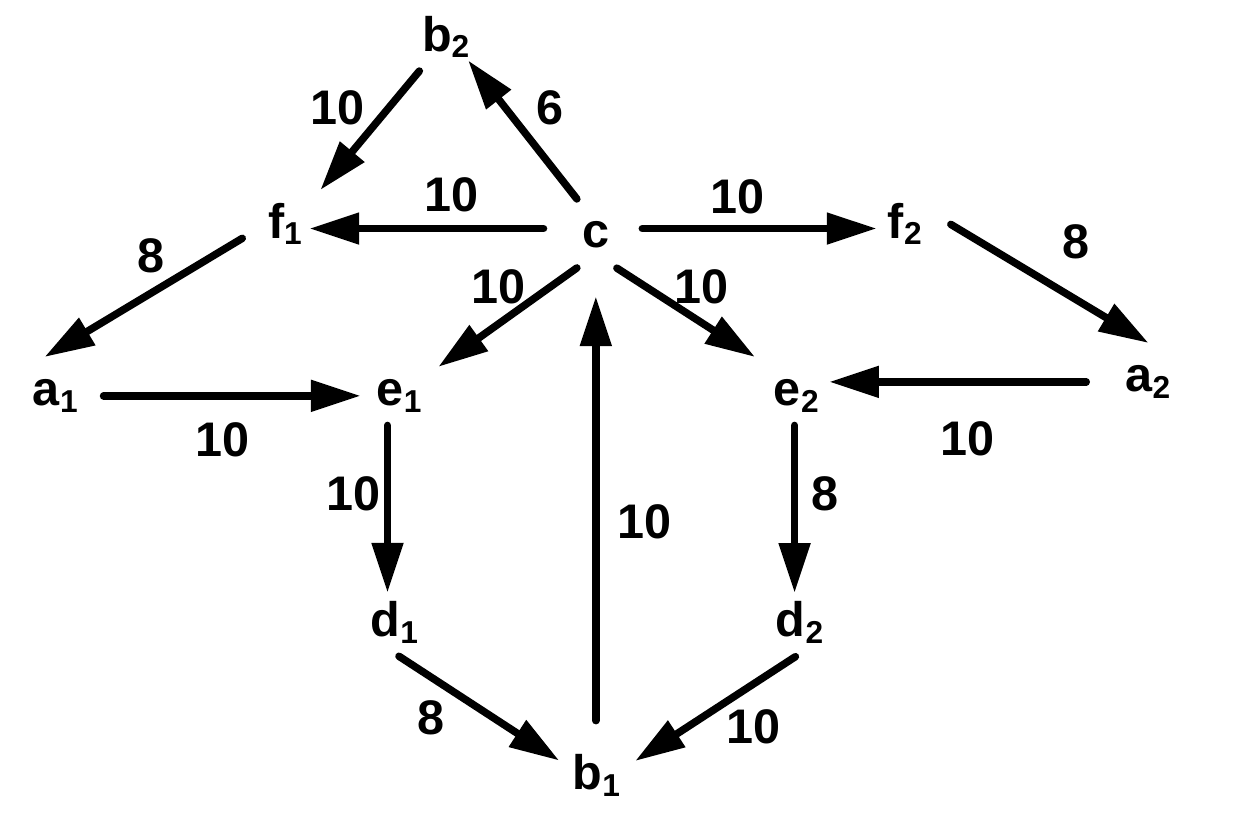}
   \caption{
       %(a)
       The WMG $\mathbf{G_P}$  from Example~\ref{exm:ucm_choices}
      % (b) The WMG $G_{P  \cup\{v\}}$ where $v$ is a valid manipulation.
       }
   \label{fig:example5}
 \end{figure}

\begin{example}\label{exm:ucm_choices}
Consider an election with eleven alternatives $\{a_1,a_2,b_1,b_2,c,d_1,d_2,e_1,e_2,f_1,f_2\}$ with the WMG in
Figure~\ref{fig:example5}, where $|M| = 1$ and $c$ is the preferred candidate.
We note that there are two candidates $b_1$ and $b_2$ such that $S_{P}(c,x) = S_{P}(x,c)-2$, $x \in \{b_1,b_2\}$.
For candidate $b_1$ there are two ways to increase $S_P(c,b_1)$. The first way is to increase the strength of the $c$-$e_1$-$d_1$-$b_1$-path by ranking  $d_1 \succ b_1$. The second way is to increase the strength of the $c$-$e_2$-$d_2$-$b_1$-path by ranking  $e_2 \succ d_2$. If we select the first way then an extension of $d_1 \succ b_1$ to any total
order leads to $c \notin W_{P'}$. If we select the second way then we can build a
successful manipulation.  We show that Algorithms~\ref{alg:bounds}--~\ref{alg:spanning_tree_ucm}
successfully construct a valid manipulation. We start with Algorithm~\ref{alg:bounds}.
Table~\ref{table:first_stage} shows  execution of Algorithm~\ref{alg:bounds} on this problem instance.
\begin{table}[tb]
  \centering
  \begin{tabular}{|c|c|c|c|c|c|c|c|c|c|c|}
    \hline
    % after \\: \hline or \cline{col1-col2} \cline{col3-col4} ...
             &   \multicolumn{10}{|c|}{Alternatives ${\cal C} \setminus \{c\}$ } \\
             \cline{2-11}
            & $f_1$ & $f_2$ & $a_1$ & $a_2$ & $e_1$ & $e_2$ & $d_1$ & $d_2$ & $b_1$ & $b_2$ \\
            \hline
            \multicolumn{11}{|c|}{Initial values} \\ \hline
   $U(c,\cdot)$  & $11$ & $11$ & $9$& $9$ & $11$ & $11$ & $11$ & $9$ & $9$ & $7$  \\
   $U(\cdot,c)$  & $9$ & $9$ & $9$ & $9$ & $9$ & $9$ & $9$ & $11$ & $11$ & $9$ \\
      \hline
   \multicolumn{11}{|c|}{Rule 1 updates $U(d_2,c)$, $U(b_1,c)$ and $U(b_2,c)$} \\
   \hline
   $U(c,\cdot)$  & $11$ & $11$ & $9$ & $9$ & $11$ & $11$ & $11$ & $9$ & $9$ & $7$  \\
   $U(\cdot,c)$  & $9$ & $9$ & $9$ & $9$ & $9$ & 9 & 9 & $\mathbf{9}$ & $\mathbf{9}$ & $\mathbf{7}$ \\
      \hline
   \multicolumn{11}{|c|}{Rule 3 updates $U(f_1,c)$ as $U(b_2,c)=7$ and $\underline w(b_2,f_1) = 9$} \\
   \hline
   $U(c,\cdot)$  & $11$ & $11$ & $9$ & $9$ & $11$ & $11$ & $11$ & $9$ & $9$ & $7$  \\
   $U(\cdot,c)$  & $\mathbf{7}$ & $9$ & $9$ & $9$ & $9$ & $9$ & $9$ & $9$ & $9$ & $7$ \\
      \hline
   \multicolumn{11}{|c|}{Rule 2 updates $U(a_1,c)$ (Figure~\ref{fig:example5ba}(a) for $G^{a_1}$)} \\
   \hline
   $U(c,\cdot)$  & $11$ & $11$ & $9$ & $9$ & $11$ & $11$ & $11$ & $9$ & $9$ & $7$  \\
   $U(\cdot,c)$  & $7$ & $9$ & $\mathbf{7}$ & $9$ & $9$ & $9$ & $9$ & $9$ & $9$ & $7$ \\
       \hline
   \multicolumn{11}{|c|}{Rule 3 updates  $U(e_1,c)$ and $U(d_1,c)$} \\
   \hline
   $U(c,\cdot)$  & $11$ & $11$ & $9$ & $9$ & $11$ & $11$ & $11$ & $9$ & $9$ & $7$  \\
   $U(\cdot,c)$  & $7$ & $9$ & $7$ & $9$ & $\mathbf{7}$ & $9$ & $\mathbf{7}$ & $9$ & $9$ & $7$ \\
    \hline
  \end{tabular}
  \caption{Execution of Algorithm~\ref{alg:bounds} on Example~\ref{exm:ucm_choices}.
  $\mathbf{U(c,\cdot)}$/$\mathbf{U(\cdot,c)}$ stands for the upper bound value $\mathbf{U(c,c')}$/$\mathbf{U(c',c)}$, where $\mathbf{c'}$ is the alternative
  in the corresponding column, $\mathbf{c' \in \ } \mathbfcal{ C} \mathbf{\setminus \{c\}}$.}\label{table:first_stage}
\end{table}
\qed
\end{example}
\begin{figure}[tb]
   \centering
   \includegraphics[width=1\linewidth]{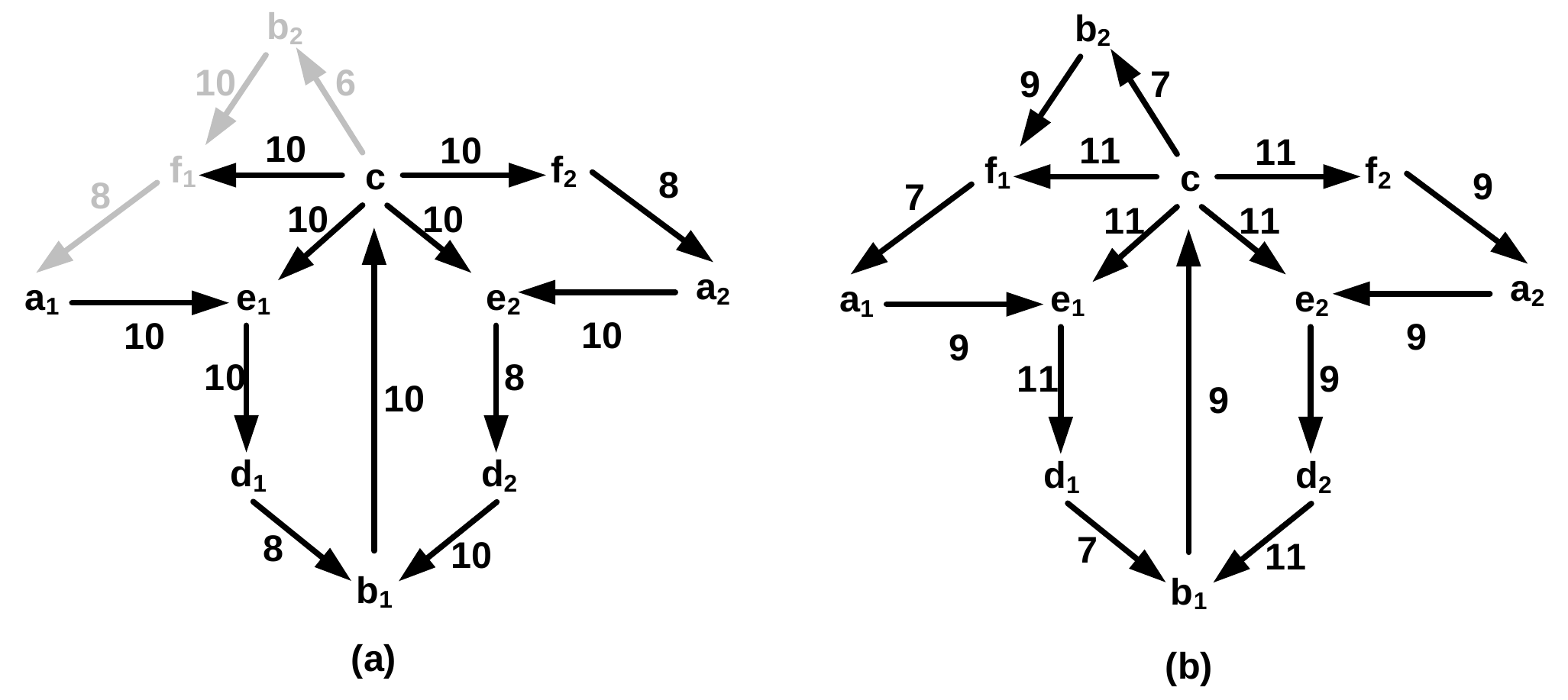}
   \caption{
       %(a)
       %The WMG $G_P$  from Example~\ref{exm:ucm_choices}
       (a) $\mathbf{G^{a_1}}$. Deleted arcs and vertices are in gray. There is no path from $\mathbf{c}$
       to $\mathbf{a_1}$; (b) WMG $\mathbf{G_{P  \cup\{\Lambda\}}}$ from Example~\ref{exm:ucm_choices}  where $\mathbf{\Lambda}$ is a valid manipulation.
       }
   \label{fig:example5ba}
 \end{figure}

\subsection{Stage 2. Construct manipulators' votes}~\label{ss:second_stage}

  \begin{algorithm}[tb]
\caption{Construction of ordering $\Lambda$}\label{alg:spanning_tree_ucm}
\vspace*{-0.4cm}
    \begin{tabbing}
          .....\=.....\=.....\=.....\=.....\=.............................. \kill \\
\textbf{Input:}\>\>\> a weighted digraph $(G=(V,E),w)=(G_P,w_P)$, \\
\>\>\> the strengths $S_P$, \\
\>\>\> a distinguished candidate $c$ and\\
\>\>\> the function $U$ returned by Algorithm \ref{alg:bounds}.\\[1ex]
\textbf{for} $(x,y)\in V\times V$ \textbf{do}\\
\> $\overline{w}(x,y)= w(x,y)+|M|$ \\
Initialize $F = \{c\}$, $X = \mathcal C \setminus \{c\}$,
%$T = \{c\}$,
$lastv = c$ and $\Lambda = \{\}$.\\
\textbf{for} $i=1,\ldots,|V|-1$ \textbf{do}\\
\> $D = \max\{U(y,c)\ :\ y\in X\}$\\
\> Choose  $x\in F$ and $y\in X$ with $\overline{w}(x,y) \geqslant U(y,c) = D$\\
\> $F= F \cup \{y\}$\\
\> $X = X \setminus \{y\} $\\
%\> $T = T \cup \{(x,y)\}$\\
\> $\Lambda = \Lambda \cup \{lastv \succ y\}  $\\
\> $lastv = y  $\\
return $\Lambda$
    \end{tabbing}
    \vspace*{-0.4cm}
  \end{algorithm}
%A subgraph $T$ of $G$ is an \emph{out-tree} if $T$ is an oriented tree with only one vertex of
%in-degree zero, called the \emph{root} of $T$. It is an \emph{out-branching} if it is an out-tree and it is spanning, i.e., if %$V(T)=V$.

Algorithm~\ref{alg:spanning_tree_ucm} constructs a linear order $\Lambda$ based on the following greedy procedure.
Initially, $\Lambda = \{\}$, $c$ is the top candidate, $lastv =c$,
the frontier $F = \{c\}$ and the set of unreached vertices $X = \mathcal C \setminus \{c\}$.
During the execution of the algorithm, $\Lambda$ is a linear order on $F$ and contains an element $x \succ y$ for any two consecutive vertices $x,y$ in this order.
The vertex $lastv$ is the last vertex in this order $c \succ \dots \succ lastv$.

While $\Lambda$ is not a total order,
the algorithm adds one of the unreached vertices $y$ to the end of a partial order $\Lambda$ satisfying
the following conditions: $x \in F$, $y \in X$,
$U(y,c)=D$ and $\overline w(x,y) \geqslant D$, where $D$ is the maximum value
$U(y,c)$ among all unreached vertices $y \in X$.

\begin{theorem}
Algorithm~\ref{alg:spanning_tree_ucm} constructs a total order $\Lambda$ with top element $c$.
Moreover, for any vertex $x\in V\setminus \{c\}$, there is a $c$-$x$-path
$\pi = (c=x_1,\ldots,x_p=x)$ such that $\overline w(x_i,x_{i+1}) \geqslant U(x,c)$
and $x_i \succ x_{i+1} \in \Lambda$, $i=1,\ldots,p-1$.
\end{theorem}
\begin{proof}
First, we need  to prove that the algorithm can always add a vertex $y$ to the order $\Lambda$ satisfying the conditions above.
Let $z$ be any candidate from $X$ such that $U(z,c)=D$.
Since Rule 2 does not apply, the subgraph $G^z$ has a directed path from $c$ to $z$.
Let $(x,y)$ be the arc on this path with $x\in F$ and $y \in X$ (we could possibly have that $y=z$).
Also, by Rule 2 we have that $U(y,c) \geqslant U(z,c)$ and that $\overline w(x,y) \geqslant U(z,c)$.
Thus, $U(y,c)=D$ and $\overline w(x,y) \geqslant D$, which means that $(x,y)$ satisfies the conditions of the alternative $y$
to be added to $\Lambda$.

We prove the second statement by induction. In the base case, $F =\{c\}$ and we add
$y$ such that $\overline{w}(c,y) \geqslant U(y,c)$.  Hence, $\pi = (c=x_1,x_2=y)$,
$\overline w(c,y) \geqslant U(y,c)$ and $c \succ y \in \Lambda$.
%The base case holds.

Suppose, the statement holds for $i-1$ steps.
Let $(x,y)$ be the arc such that  $x\in F$ and $y \in X$,
$\overline w(x,y) \geqslant U(y,c)= D$ that we
add at the $i$th step. By the induction hypothesis, we know that
there is a $c$-$x$-path
$\pi = (c=x_1,\ldots,x_{p}=x)$ such that $\overline w(x_j,x_{j+1}) \geqslant U(x,c)$
and $x_j \succ x_{j+1} \in \Lambda$, $j=1,\ldots,p-1$, $p \leqslant i-1$.
By the selection of $y$, we get that $\overline w(x,y) \geqslant U(y,c)$.
By Algorithm~\ref{alg:spanning_tree_ucm}, we know that $U(x,c) \geqslant U(y,c)$.
Hence,  $\overline w(x_j,x_{j+1}) \geqslant U(x,c) \geqslant U(y,c)$, $j=1,\ldots,p-1$.
As we add $x \succ y$ to $\Lambda$ we get that there is a $c$-$y$-path
$\pi = (c=x_1,\ldots,x_{p}=x,x_{p+1}=y)$ such that $\overline w(x,y) \geqslant U(y,c)$
and $x_j \succ x_{j+1} \in \Lambda$, $j=1,\ldots,p$.
\end{proof}

%Let $\sigma$ denote the order in which vertices were added to $T$ returned by Algorithm~\ref{alg:spanning_tree_ucm},
%i.e., $\sigma(x) < \sigma(y)$ if and only if $x$ was added to $T$ before $y$.
This order $\Lambda$ defines the vote $\succ '$ of the manipulators.
%We denote $\Lambda(x)$ the step number when $x$ is added to $\Lambda$.

\begin{theorem}\label{thm:ucm_manip}
Consider the order $\Lambda$ returned by Algorithm~\ref{alg:spanning_tree_ucm}.
Then  $c \in W_{P'}$, where $ P' = P^{NM} \cup \bigcup_{i=1}^{|M|} \{\Lambda\}$.
\end{theorem}
\begin{proof}
%Since candidate $c$ is ordered first, we have that $S_{P'}(c,x)=U(c,x)$ for every candidate $x\in V\setminus \{c\}$.
Due to the construction of $\Lambda$, we know that $S_{P'}(c,x) \geqslant U(x,c)$, $x\in V\setminus \{c\}$ as for each vertex $x$
there is a $c$-$x$-path $\sigma=(c=x_1,x_2,\dots,x_p=x)$ such that $\overline w(x_i,x_{i+1}) \geqslant U(x,c)$
and $x_i \succ x_{i+1} \in \Lambda$, $i=1,\ldots,p-1$.

Let us make sure that $S_{P'}(x,c) \leqslant U(x,c)$ for each candidate $x\in V\setminus \{c\}$. On the contrary, suppose there is a candidate $x$ such that $S_{P'}(x,c) > U(x,c)$ and suppose among all such vertices, $x$ has the shortest critical path to $c$. Denote by $\pi=(x,x_1,x_2,\dots,c)$ a shortest critical path from $x$ to $c$. Consider two cases depending on whether $x_1 =c$ or $x_1 \neq c$.

Suppose that $x_1 \neq c$.
We have that $S_{P'}(x_1,c) \geqslant S_{P'}(x,c)$ since the path $\pi$ is critical.
Therefore, $U(x_1,c) > U(x,c)$ by the selection of $x$ and $\pi$.
Since candidates are added by non-increasing values of $U(\cdot,c)$ to $\Lambda$,
$x_1$ was added before $x$, so that $x_1\succ' x$.
%$\Lambda(x_1) > \Lambda(x)$.
Thus, $w'(x,x_1) = \underline w(x,x_1)$.
By Reduction Rule 3, we have that $\underline w(x,x_1) \leqslant U(x,c)$.
Thus, $w'(x,x_1) \leqslant U(x,c)$, contradicting that $\pi$ has strength $> U(x,c)$ in $G'$.

Suppose that $x_1 = c$. In this case, $\pi = (x,c)$ and $S_{P'}(x,c) = w'(x,c)$.
As $c$ is the top element of $\Lambda$ we have that $ w'(x,c) = \underline w(x,c) = w(x,c) -|M|$.
As Algorithm~\ref{alg:bounds} did not detect a failure, we know that
$U(x,c) \geqslant S_P(x,c)-|M|$. Moreover, $S_P(x,c) \geqslant w(x,c)$,
by definition of the critical path. Therefore,
$U(x,c) \geqslant S_P(x,c)-|M| \geqslant  w(x,c) -|M| = w'(x,c) = S_{P'}(x,c)$.
Hence, $S_{P'}(x,c) \leqslant U(x,c)$, contradicting that $\pi$ has strength $> U(x,c)$ in $G'$.
\end{proof}

{Note that Corollary~\ref{col:wcm} does not follow from Theorem~\ref{thm:ucm_manip}, because Algorithm~\ref{alg:bounds} takes $O(w_{max}|V|^3)$ time, where $w_{max} = \max_{(x,y) \in V \times V}w(x,y)$. As $w_{max}$ can be $O(2^{|V|})$,
Algorithm~\ref{alg:bounds} takes exponential number of steps in WCM.}

%\begin{figure}[tb]
%   \centering
%   \includegraphics[width=0.7\linewidth]{example5_2}
%   \caption{
%       %(a)
%       %The WMG $G_P$  from Example~\ref{exm:ucm_choices}
%        WMG $\mathbf{G_{P  \cup\{\Lambda\}}}$ from Example~\ref{exm:ucm_choices}  where $\mathbf{\Lambda}$ is a valid manipulation.
%       }
%   \label{fig:example5a}
% \end{figure}
\begin{example}
Consider how Algorithm~\ref{alg:spanning_tree_ucm} works on Example~\ref{exm:ucm_choices}.
The algorithm traverses $G$ by vertices ordered by the value $U(x,c)$, $x\in \mathcal C  \setminus\{c\}$.
Initially, we start at $c$, and $F = \{c\}$ and $X = \mathcal C \setminus \{c\}$. We
compute $D = \max\{U(y,c)\ :\ y\in X\}$, $D = 9$.
We consider all vertices $y \in X$ such that $U(y,c) = 9$,
which is the set $Q = \{f_2,a_2,e_2,d_2, b_1\}$.
We select one of those vertices, $f_2$, that satisfies the condition on the value $\overline w(x,y)$, $x \in F$, $y \in X$:
$\overline w(c,f_2) = 11 \geqslant U(f_2,c) = 9$.
In the next four steps we add all elements of $Q$ and obtain a partial order $\Lambda = c \succ f_2 \succ a_2 \succ e_2 \succ b_1$.
The next maximum value $D = \max\{U(y,c)\ :\ y\in \mathcal C \setminus \{c,f_2,a_2,e_2,d_2,b_1\}$ is $7$.
The set of vertices such that $U(y,c)=7$ is $Q=\{f_1,a_1,e_1,d_1,b_2\}$. Hence, we add
these vertices to $\Lambda$ one by one and obtain a total order
$\Lambda = c \succ f_2 \succ a_2 \succ e_2 \succ b_1 \succ f_1 \succ a_1 \succ b_2 \succ e_1 \succ d_1$.
Figure~\ref{fig:example5ba}(b) shows the WMG  $G_{P \cup \{\Lambda\}}$. We omitted all arcs
of weight 1 for clarity.\qed
\end{example}

\section{Unique winner vs co-winner UCM}\label{sec:UCM}
In this section we consider the unweighted coalitional manipulation problem with a single manipulator that was considered in~\cite{rankedpairs2}. Parkes and Xia showed that the \emph{unique winner UCM} for Schulze's rule with a single manipulator can be solved in polynomial time. We emphasize that in this variant the aim is to make the preferred candidate $c$ the unique winner.
The aim of this section is to show that the proof from~\cite{rankedpairs2} cannot be extended to the co-winner UCM problem with one manipulator. This demonstrates that the co-winner UCM problem with one manipulator was not resolved in \cite{rankedpairs2}. We also extend our algorithm for co-winner UCM to the unique winner case.
Another reason to investigate the relation between properties of unique winner and co-winner manipulation problems is that
they are closely related to the choice of tie-breaking rules.
If the tie-breaking rule breaks ties against the manipulators then the
manipulators have to ensure that the preferred candidate
is the unique winner of an election. If the tie-breaking rule breaks ties in favor of the manipulators
then it is sufficient for the manipulators to guarantee that the preferred candidate is one of the
co-winners of the election to achieve the desired outcome.

The proof that the unique winner UCM is polynomial is based on the resolvability property \cite[Section 4.2.2]{Schulze2011}. The resolvability criterion states that any co-winner can be made a unique winner by adding a single vote.
\begin{description}
\item[Resolvability.] If $S_P(c,x)\geqslant S_P(x,c)$ for all candidates $x\in\mathcal C\setminus\{c\}$, then there is a vote $v$ such that $S_{P\cup\{v\}}(c,x)>S_{P\cup\{v\}}(x,c)$ for all candidates $x\in\mathcal C\setminus\{c\}$.
\end{description}
The proof of the property is constructive. Clearly, $c$ can be the unique winner in $P \cup \{v\}$ only if $c$ is a co-winner in $P$. The vote $v$ is constructed using two rules that we describe below. We denote $P = P^{NM}$ and $\{v\} =P^{M}$ to simplify notations.
\begin{description}
\item[(1)] For every alternative $x \in \mathcal \mathcal C \setminus \{c\}$, we require $y\succ x$ in the manipulator's vote $v$ where $y$ is the predecessor of $x$ on some strongest path from $c$ to $x$.
\item[(2)] For any $x,y \in \mathcal C \setminus \{c\}$ with $S_{P}(x,c) > S_{P}(y,c)$ we require $x \succ y$ in the manipulator's vote $v$.
\end{description}

It was shown in~\cite{Schulze2011} that the resulting set of preference relations does not contain cycles and thus can be extended to a linear order which makes $c$ the unique winner. However, it was also shown in~\cite{Schulze2011} that the same approach cannot resolve ties between candidates that do not belong to the winning set. It is a natural question if a candidate $c$ that is not in the winning set can be made a winner by adding a single vote. Clearly a necessary condition is $S_P(c,x) \geqslant S_P(x,c)-2$ for all $x \in \mathcal C \setminus\{c\}$. So we can formulate the following problem.
\begin{description}
\item[Single vote UCM.] Given a profile $P$ and a candidate $c$ with $S_P(x,c) \leqslant S_P(c,x)+2$ for all $x \in \mathcal C \setminus\{c\}$, does there exist a single vote $v$ such that
$c\in \mathcal W_{P\cup \{v\}}$?
%$S_{P\cup\{v\}}(c,x)\geqslant S_{P\cup\{v\}}(x,c)$ for all $x\in\mathcal C\setminus\{c\}$?
\end{description}
Here, we show that the straightforward adaption of the above rules does not solve this problem, even if there is a single vote manipulation that makes $c$ a winner.
 \begin{figure}[tb]
   \centering
   \includegraphics[width=0.8\linewidth]{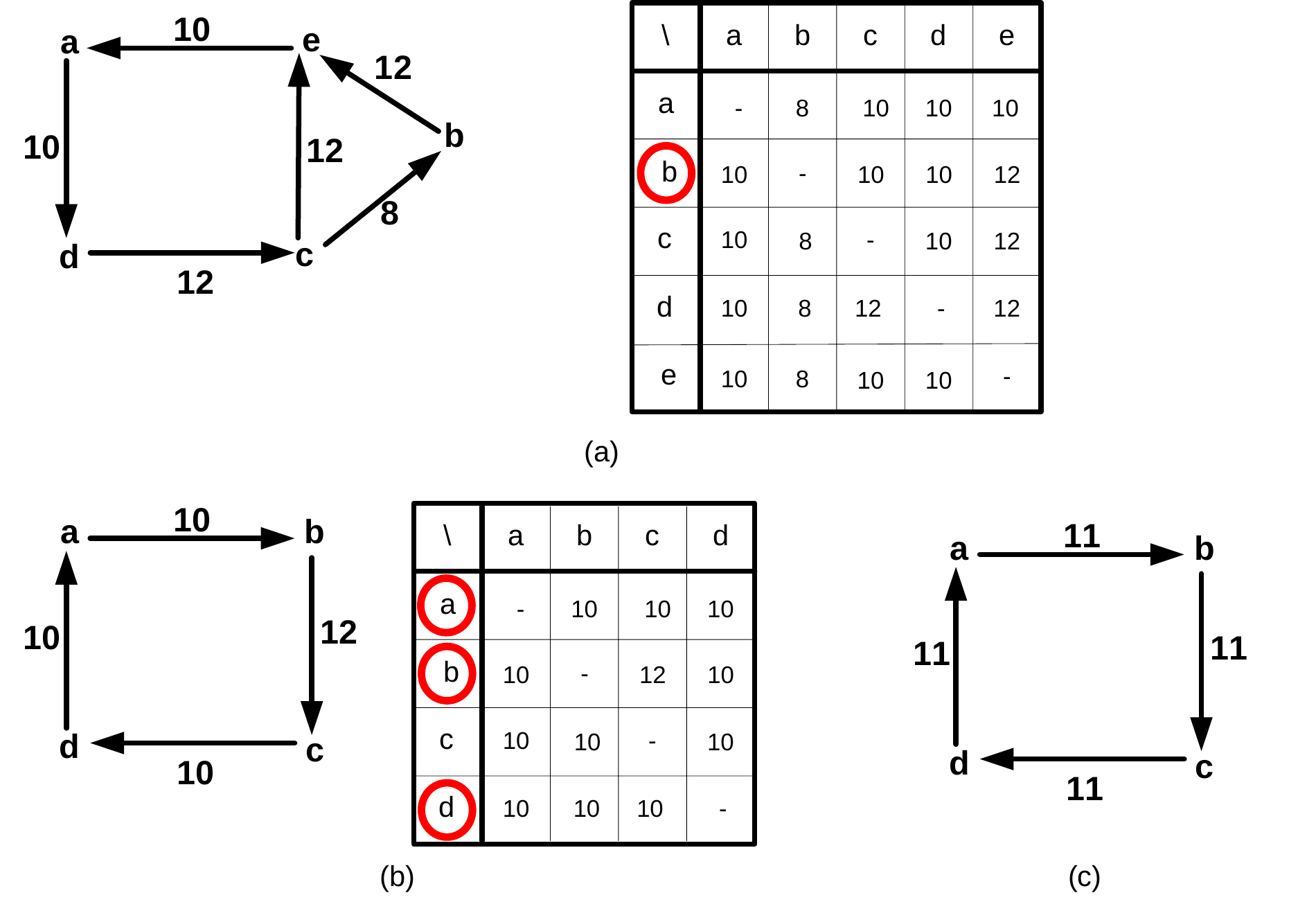}
   \caption{
      (a) The WMG $\mathbf{G_P}$ and the table of $\mathbf{S_{P}(x,y)}$, $\mathbf{x,y \in \{a,b,c,d,e\}}$ from Example~\ref{exm:ucm_no_winner};
   (b)/(c) The WMG $\mathbf{G_P / G_{P\cup\{v\}}}$ and the table of $\mathbf{S_{P}(x,y)}$/$\mathbf{S_{P \cup\{v\}}(x,y)}$, $\mathbf{x,y \in \{a,b,c,d\}}$ from Example~\ref{exm:ucm_cycles}.}
   \label{fig:example3}
 \end{figure}
A major difference between the unique winner and the co-winner UCM problems is that the manipulation always exists in the former problem and it might not exist in the latter as the following example demonstrates.

\begin{example}\label{exm:ucm_no_winner}
Consider an election with five alternatives $\{a,b,c,d,e\}$. Figure~\ref{fig:example3}(a) shows the WMG and the corresponding table of maximum strengths. The unique winner is $b$.  However, the difference $S_P(x,c) - S_P(c,x) \leqslant 2$, $x \in \{a,b,d,e\}$. Hence, $c$ satisfies the trivial necessary condition for being made a winner by adding a single vote.

To see that there is no successful manipulation we  notice that $S_P(c,d)=S_P(d,c)-2$. Hence the manipulation must increase the weight of at least one critical $c$-$d$-path. As there is only one critical $c$-$d$-path this forces $e\succ a\succ d$ in the manipulator's vote. But on the other hand $S_P(c,b)=S_P(b,c)-2$ requires that the weight of every critical $b$-$c$-path decreases which implies that $a\succ e$ or $d\succ a$, which gives a contradiction.

Consider the preference relations that are output by the rules. Following the first rule we add $e \succ a \succ d$ and $c \succ b$. Following the second rule, we add $d \succ \{a,b,e\}$. This creates a cycle and thus cannot be completed to a linear order.\qed
\end{example}

Next, we show that the rules do not find the manipulator vote even if such a manipulation exists for the co-winner UCM problem
using Examples~\ref{exm:ucm_cycles}--\ref{exm:ucm_cycles2}.

\begin{example}\label{exm:ucm_cycles}
Consider an election with four alternatives $\{a,b,c,d\}$. Figure~\ref{fig:example3}(b) shows its WMG and the corresponding table of maximum strengths. The set of winners is $\{a,b,d\}$ and $S_P(x,c) - S_P(c,x) \leqslant 2$, $x \in \mathcal W_P$. Following the first rule we add $c \succ d \succ a \succ b$. However, by the second rule, we add $b \succ a$ which creates a cycle. Note that a successful manipulation $v$ exists $v = (c \succ d \succ a \succ b)$ (Figure~\ref{fig:example3}(c)).\qed
\end{example}

%Another example that demonstrates that the rules do not find the manipulator vote even if such a manipulation exists for the %co-winner UCM problem is our running example from Section~\ref{sec:UCM}.

\begin{example}\label{exm:ucm_cycles2}
Consider the election with 11 alternatives from Example~\ref{exm:ucm_choices}.
Following the first rule we add $c \succ e_1 \succ d_1 \succ b_1$ to the manipulator vote
as $\pi = (c, e_1, d_1, b_1)$ is a strongest path from $c$ to $b_1$.
As we showed in Example~\ref{exm:ucm_choices}, there does not exist an extension of this partial
order to a total order that makes $c$ a co-winner.  However, a successful manipulation $v$ exists (Figure~\ref{fig:example5ba}(b)).\qed
\end{example}

Therefore, our study highlights a difference between unique winner and co-winner UCM under Schulze's rule with
a single manipulator and demonstrates that co-winner UCM with a single manipulator was not resolved.
Moreover, we believe that Schulze's rule is an interesting example, where the tie-breaking in favor of a manipulator, which
corresponds to co-winner UCM, makes the problem non-trivial compared to  tie-breaking against manipulators,
which corresponds to unique winner UCM. Two rules with similar properties have been considered in the literature.
Conitzer, Sandholm and Lang~\cite{conitzer2007elections} showed that Copeland's rule is polynomial
with 3 candidates in unique winner WCM, while it is NP-hard with 3 candidates in co-winner WCM~\cite{copeland}.
The most recent result is due to Hemaspaandra, Hemaspaandra and Rothe~\cite{onlineman} who showed that
the online manipulation WCM is polynomial for plurality in the co-winner model,
while it is  coNP-hard in the unique winner model.

Our algorithm from Section~\ref{sec:UCM} can still be used as a subroutine to solve the unique winner UCM problem.
\begin{corollary}
 The unique winner UCM problem can be solved in polynomial time.
\end{corollary}
\begin{proof}
 Run the algorithm from Section~\ref{sec:UCM} with $|M|-1$ manipulators and return the answer.
 To show the correctness of this procedure, we need to show that $c$ is a co-winner with $|M|-1$ manipulators iff $c$ is a unique winner with $|M|$ manipulators.

 $(\Rightarrow)$: Suppose $c$ can be made a co-winner with $|M|-1$ manipulators. Use the Resolvability property to add one more vote to make $c$ a unique winner.

 $(\Leftarrow)$: Suppose $c$ can be made a unique winner with $|M|$ manipulators. Therefore, $S_P(c,x) \ge S_P(x,c)+2$ for every candidate $x\in \mathcal{C}\setminus \{c\}$ in the profile $P=P^{NM} \cup P^{M}$. Now, remove an arbitrary vote of a manipulator and obtain the profile $P'$. We have that $S_{P'}(c,x) \ge S_P(c,x)-1$ and $S_{P'}(x,c) \le S_P(x,c)+1$ for every candidate $x\in \mathcal{C}\setminus \{c\}$. Therefore, $S_{P'}(c,x) \ge S_P(c,x)-1 \ge S_P(x,c)+1 \ge S_{P'}(x,c)$ for every candidate $x\in \mathcal{C}\setminus \{c\}$, showing that $c$ is a co-winner with $|M|-1$ manipulators.
\end{proof}
% Similarly, the unique winner WCM problem can be solved in polynomial time if the number of candidates is bounded.
% \begin{corollary}
%  The unique winner WCM problem can be solved in polynomial time if the number of candidates is bounded.
% \end{corollary}

\section{Conclusions}
We have
investigated the computational complexity of the coalitional weighted and unweighted manipulation problems
under Schulze's rule.
We  proved that it is polynomial to
manipulate Schulze's rule with any number of  manipulators
in the weighted co-winner model and in the unweighted case in both unique and co-winner models.
This resolves an open question regarding
the computational complexity of unweighted coalition
manipulation for Schulze' rule~\cite{rankedpairs2}.
This vulnerability to manipulation
may be of concern to the many supporters of
Schulze's rule.
%
%
% \medskip
% \noindent\textbf{Acknowledgments.}
% NICTA is funded by
% the Australian Government as represented by
% the Department of Broadband, Communications and the Digital Economy and
% the Australian Research Council.
% Serge Gaspers acknowledges support from the Australian Research Council (grant DE120101761).

\section{Acknowledgments}
NICTA is funded by
the Australian Government as represented by
the Department of Broadband, Communications and the Digital Economy and
the Australian Research Council. This research is also
funded by AOARD grant 124056. Serge Gaspers acknowledges support from
 the Australian Research Council (grant DE120101761).

%\bibliographystyle{plain}
%\bibliography{lit}

\begin{thebibliography}{10}

\bibitem{AHUJA93}
R.K. Ahuja, T.L. Magnanti, and J.B. Orlin.
\newblock {\em Network Flows: Theory, Algorithms and Applications}.
\newblock Prentice Hall, 1993.

\bibitem{stvhard}
J.J. Bartholdi and J.B. Orlin.
\newblock Single transferable vote resists strategic voting.
\newblock {\em Social Choice and Welfare}, 8(4):341--354, 1991.

\bibitem{bartholditoveytrick}
J.J. Bartholdi, C.A. Tovey, and M.A. Trick.
\newblock The computational difficulty of manipulating an election.
\newblock {\em Social Choice and Welfare}, 6(3):227--241, 1989.

\bibitem{borda2}
N.~Betzler, R.~Niedermeier, and G.J. Woeginger.
\newblock Unweighted coalitional manipulation under the {Borda} rule is
  {NP}-hard.
\newblock In {\em Proc. of 22nd International Joint
  Conference on AI}. 2011.

\bibitem{conitzer2007elections}
V.~Conitzer, T.~Sandholm, and J.~Lang.
\newblock When are elections with few candidates hard to manipulate?
\newblock {\em Journal of the ACM (JACM)}, 54(3):14, 2007.


\bibitem{dknwaaai11}
J.~Davies, G.~Katsirelos, N.~Narodytska, and T.~Walsh.
\newblock Complexity of and algorithms for {Borda} manipulation.
\newblock In {\em Proc. of 25th AAAI Conference on AI}. 2011.

\bibitem{Debord87PhD}
B.~Debord.
\newblock {\em Axiomatisation de proc{\'e}dures d'agr{\'e}gation de pr{\'e}f{\'e}rences}.
\newblock PhD thesis, Universit{\'e} Scientifique, Technologique et M{\'e}dicale de Grenoble, 1987.

\bibitem{onlineman}
L.~Hemaspaandra, E.~Hemaspaandra, and J.~Rothe.
\newblock The Complexity of Online Manipulation of Sequential Elections
\newblock In {\em Proc. of the 5rd
Int. Workshop on Computational Social Choice
(COMSOC-12)}.

\bibitem{copeland}
P.~Faliszewski, E.~Hemaspaandra, and H.~Schnoor.
\newblock Copeland voting: ties matter.
\newblock In
  {\em Proc. of
  7th Int. Joint Conference on Autonomous Agents and Multiagent
  Systems}, %pages 983--990,
  2008.

\bibitem{survey1}
P.~Faliszewski, and A.~Procaccia.
\newblock AI's War on Manipulation: Are We Winning?
\newblock {\em AI Magazine}, 31(4):53--64, 2010.



\bibitem{Gibbard73}
A.~Gibbard.
\newblock Manipulation of voting schemes: a general result.
\newblock {\em Econometrica}, 41(4):587--601, 1973.

\bibitem{McGarvey}
D.C.~McGarvey.
\newblock A theorem on the construction of voting paradoxes.
\newblock {\em Econometrica}, 21:608--610, 1953.

\bibitem{nwxaaai11}
N.~Narodytska, T.~Walsh, and L.~Xia.
\newblock Manipulation of {Nanson's} and {Baldwin's} rules.
\newblock In {\em Proc. of
  25th AAAI Conference on AI}. 2011.

\bibitem{rankedpairs2}
D.C.~Parkes and L.~Xia.
\newblock A complexity-of-strategic-behavior comparison between {S}chulze's rule
  and ranked pairs.
\newblock In {\em Proc. of 26th AAAI Conference on
  AI}, 2012.

\bibitem{Satterthwaite75}
M.A.~Satterthwaite.
\newblock Strategy-proofness and {A}rrow's conditions: Existence and
  correspondence theorems for voting procedures and social welfare functions.
\newblock {\em Journal of Economic Theory}, 10:187--217, 1975.

\bibitem{Schulze2011}
M.~Schulze.
\newblock A new monotonic, clone-independent, reversal symmetric, and
  {C}ondorcet-consistent single-winner election method.
\newblock {\em Social Choice and Welfare}, 36:267--303, 2011.

\bibitem{rankedpairs1}
L.~Xia, M.~Zuckerman, A.D.~Procaccia, V.~Conitzer, and
  J.S.~Rosenschein.
\newblock Complexity of unweighted coalitional manipulation under some common
  voting rules.
\newblock In {\em Proc. of
21st Int. Joint
  Conference on AI}. 2009.

\end{thebibliography}

\end{document}